\documentclass[letterpaper,10pt,conference]{ieeeconf}
\IEEEoverridecommandlockouts       
\usepackage{amssymb}
\usepackage{amsmath}

\usepackage{amsthm}
\usepackage[english]{babel}
\usepackage{float}
\usepackage{graphicx}
\usepackage{hyperref}
\usepackage{mathtools}
\usepackage{booktabs}
\usepackage{afterpage}
\usepackage{dblfloatfix}
\usepackage{colortbl}

\usepackage{filecontents}
\usepackage[noadjust]{cite}
\usepackage[font=footnotesize,labelfont=footnotesize]{subcaption}
\usepackage[font=footnotesize,labelfont=footnotesize]{caption}
\usepackage{comment}
\usepackage{authblk}
\usepackage{multirow}
\usepackage{xcolor}
\usepackage{algorithm}
\usepackage{algorithmic}
\usepackage{xspace}
\usepackage{graphicx}
\usepackage[english]{babel}

\DeclareMathOperator*{\argmin}{arg\,min}

\definecolor{darkblue}{rgb}{0.15,0.15,0.55}
\definecolor{lightgrey}{rgb}{0.75,0.75,0.75}

\newtheorem{theorem}{Theorem}

\newtheorem{lemma}[theorem]{Lemma}
\theoremstyle{definition}
\newtheorem{definition}{Definition}

\newcommand{\algcomment}[1]{%
    \nobreak\hspace{1em}%
    {\scriptsize\color{gray}\textsf{// #1}}%
}

\begin{document}
\title{
\LARGE \bf 
{A Capability-Aware Backup-Recoverability Filter\\ for Heterogeneous Multi-Robot Navigation}
\vspace{-12pt}}

\author{Joonkyung Kim$^1$, Yanze Zhang$^2$, Wenhao Luo$^1$, Yiwei Lyu$^{1,*}$
\thanks{$^{1}$Department of Computer Science and Engineering, Texas A\&M University, $^{2}$Department of Computer Science, University of Illinois Chicago, $^*$Corresponding author: {\tt\small yiweilyu@tamu.edu}}
}


\maketitle 

\begin{abstract}

Safe navigation in heterogeneous multi-robot teams faces two coupled
challenges: differences in robots' braking and turning capabilities
limit how each robot can respond to an approaching neighbor, while
simultaneous interactions require a single control action to satisfy
multiple collision-avoidance constraints.
Together, these challenges can make local safety constraints
incompatible with a robot's admissible controls, leaving a
decentralized safety controller without a feasible action.
We propose a Capability-Aware Backup-Recoverability Filter (CA-BRF),
based on the insight that navigation actions should preserve access to
a certified recovery maneuver.
First, we construct a backup-recoverable set in a shared Cartesian
operational space by evaluating whether prescribed backup maneuvers
maintain pairwise safety until the robots come to rest.
Each maneuver respects the corresponding robot’s kinematic and actuation constraints, making the recovery certificate capability-dependent. Second, CA-BRF filters candidate controls by selecting admissible
interpolations between candidate and backup actions while coordinating
the selection across simultaneous interactions.
This preserves backup recoverability while retaining as much of the
candidate action as possible and ensures that all pairwise recovery
conditions are satisfied by a single executable control for each robot.
Together, we establish recursive feasibility and safety for the discrete-time closed-loop system even when the nominal safety-constrained optimization becomes infeasible. The validity and effectiveness of the proposed CA-BRF are validated in simulations with up to 30 heterogeneous robots and experiments of robots with five different kinematic classes. \vspace{-3pt}

\end{abstract}

\section{Introduction}
\label{sec:intro}
\vspace{-0.25em}

\begin{figure}[!t]
    \captionsetup{skip=2pt}
    \centering

    \includegraphics[width=0.35\columnwidth]{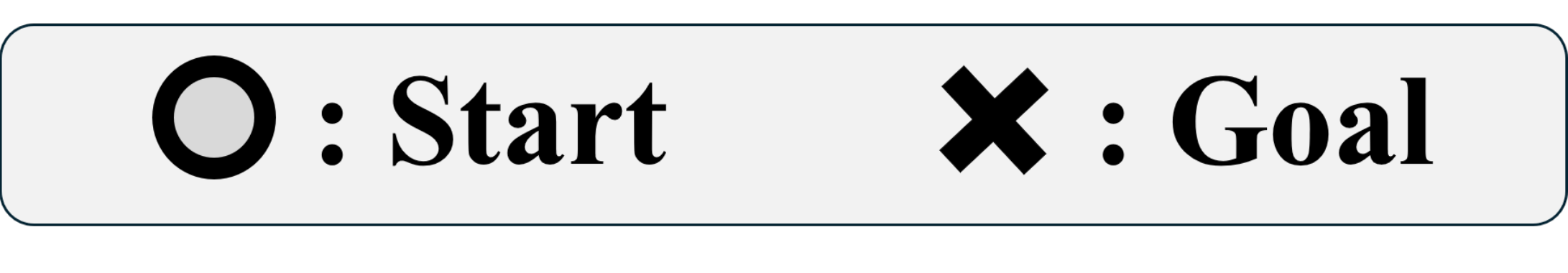}
    \vspace{0.0em}

    \begin{subfigure}[t]{0.38\columnwidth}
        \centering
        \includegraphics[width=\linewidth]{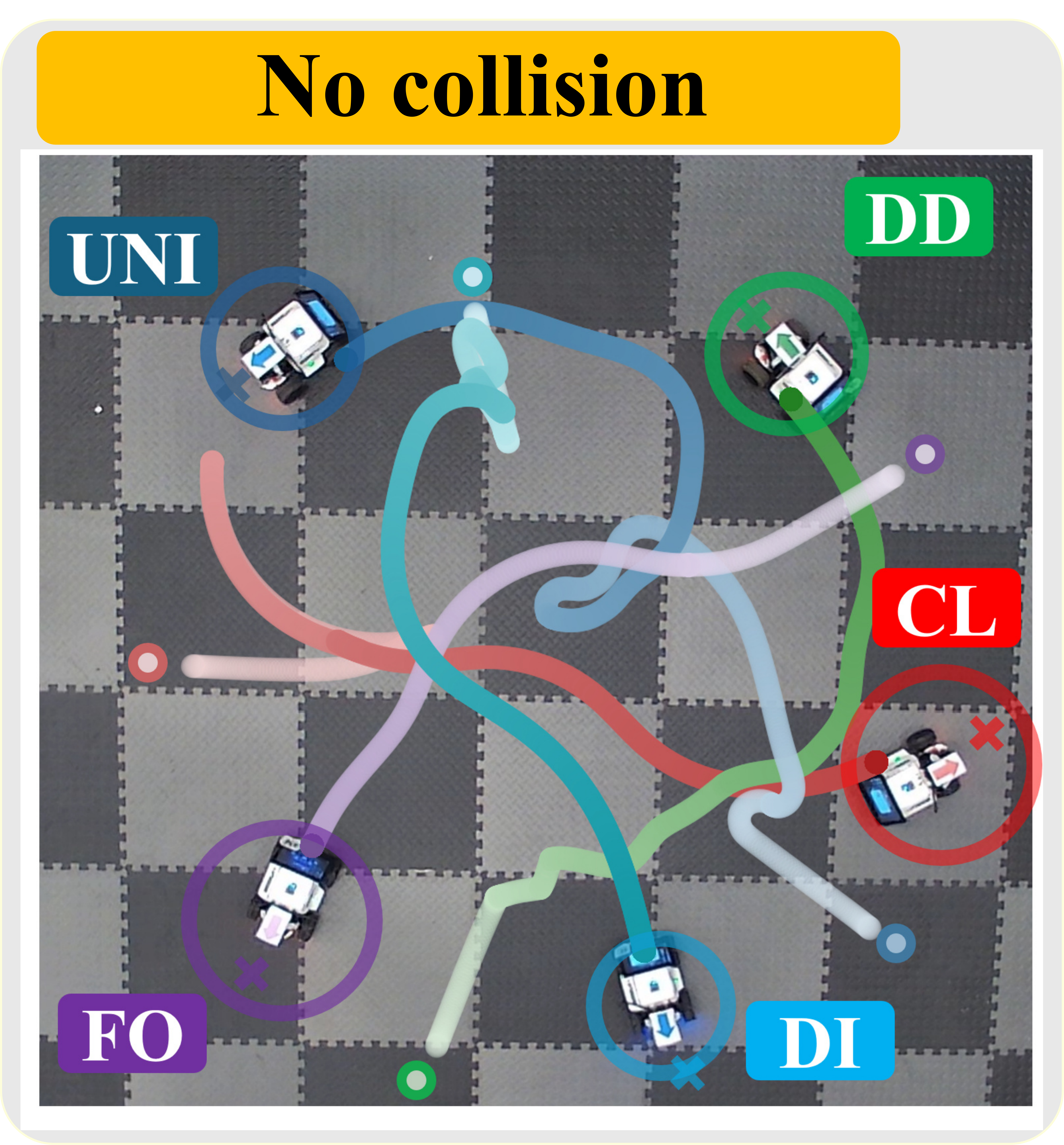}
        \caption{CA-BRF (Ours).}
        \label{fig:real_robot_cabrf}
    \end{subfigure}\hspace{6pt}
    \begin{subfigure}[t]{0.38\columnwidth}
        \centering
        \includegraphics[width=\linewidth]{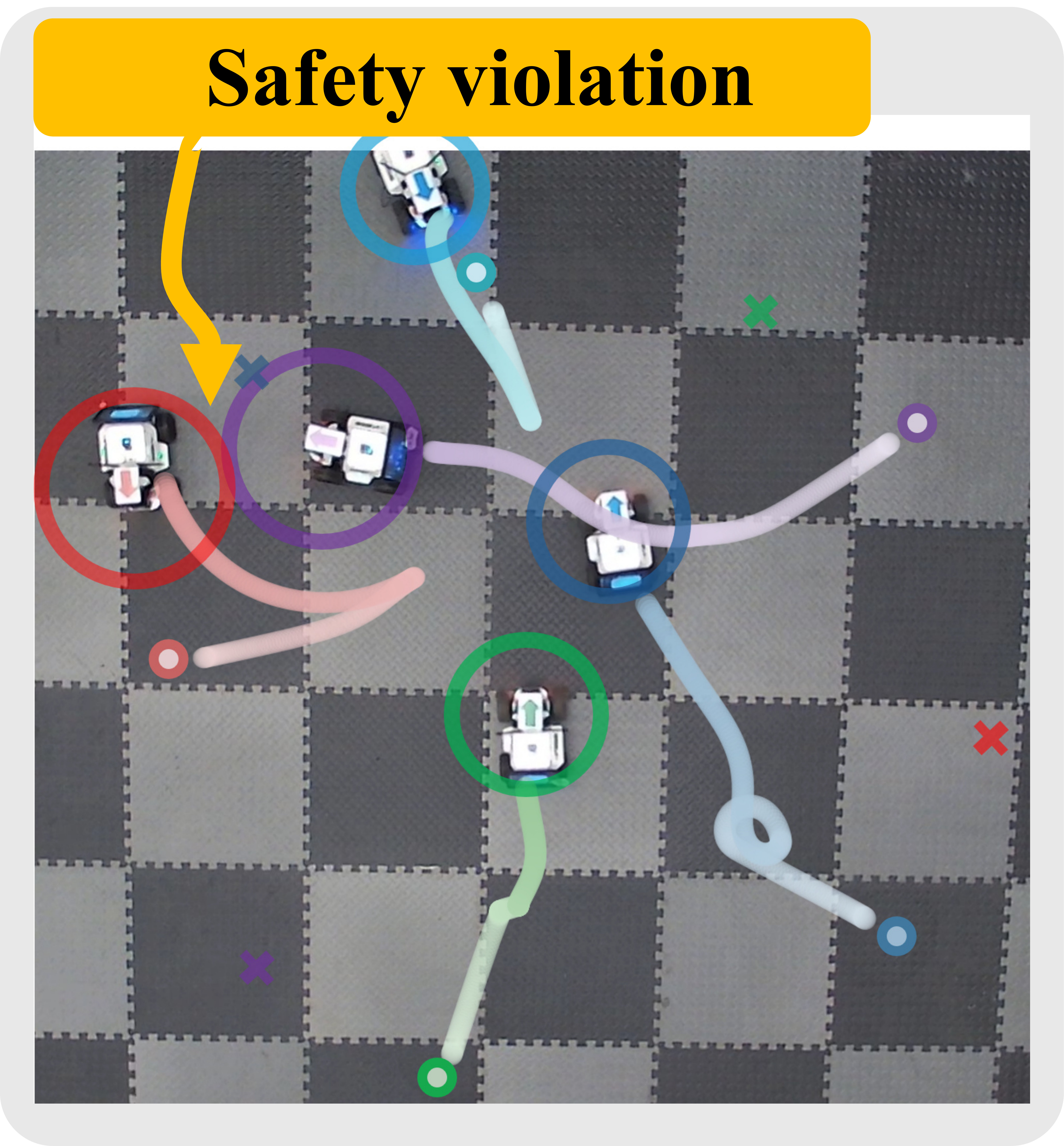}
        \caption{HOCBF only (Baseline).}
        \label{fig:real_robot_hocbf}
    \end{subfigure}

    \vspace{0.2em}

    \begin{subfigure}[t]{\columnwidth}
        \centering
        \includegraphics[width=0.99\linewidth]{figures/fig_1_c.pdf}
        \caption{Backup first-layer margin $B_{ij}^{\psi}$ for the CL--FO pair (Ours).}
        \label{fig:new_panel}
    \end{subfigure}

    \vspace{0.2em}

    \begin{subfigure}[t]{\columnwidth}
        \centering
        \includegraphics[width=0.99\linewidth]{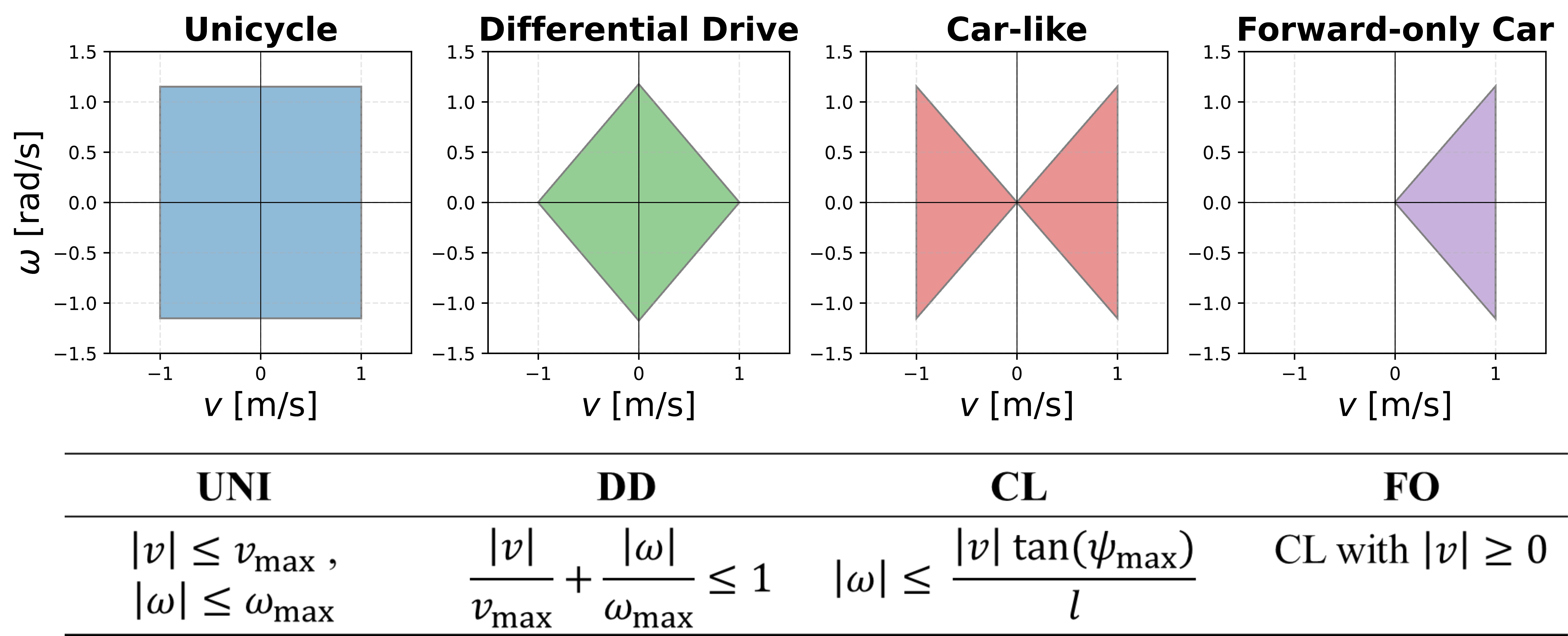}
        \caption{Feasible $(v,\omega)$ control sets for heterogeneous models.}
        \label{fig:control_sets}
    \end{subfigure}

    \caption{
    Real-robot trajectory comparison, recovery-certificate margin, and model-dependent velocity sets. Five heterogeneous robots (DI, UNI, DD, CL, and FO) navigate to individual goals. \textbf{(a)} With CA-BRF (ours), the robots navigate successfully while maintaining the prescribed pairwise safety condition. \textbf{(b)} Directly applying the HOCBF candidate controls without CA-BRF violates this condition. Darker trajectory colors indicate later times. \textbf{(c)} Backup safety margin $B_{ij}^{\psi}(\mathbf{x})$ for the CL--FO pair under the candidate and CA-BRF applied controls. Values below zero indicate that the prescribed backup maneuvers no longer certify safe separation. The candidate control crosses below zero, whereas the CA-BRF applied control remains nonnegative, preserving backup recoverability. The shaded interval marks the active CL--FO interaction during which CA-BRF modifies the candidate control. \textbf{(d)} Feasible $(v,\omega)$ velocity sets and constraints for four representative nonholonomic kinematic classes~\cite{lynch2017modern_robotics}; DI is omitted as holonomic.  
    }
    \label{fig:overview} \vspace{-2.5em}
\end{figure}

Heterogeneous robot teams combine complementary capabilities for
applications such as logistics~\cite{li2021lifelong},
exploration~\cite{gao2020random}, and inspection~\cite{tranzatto2022cerberus}.
Safe navigation in shared workspaces faces two coupled challenges:
kinematic and actuation constraints restrict each robot's available
responses~\cite{wang2016safety, agrawal2021safe}, while simultaneous
encounters require compatible actions across several collision-avoidance
constraints. For example, braking may help a car-like robot avoid a
robot ahead, but its resulting loss of turning authority can limit its
response to another robot crossing its path~\cite{deluca2002control};
a differential-drive robot has different options because it can rotate
in place. Safe coordination must therefore account for both each
robot's state-dependent physical capability and the compatibility of
actions selected across its interactions.


Control barrier functions (CBFs) provide a principled approach to
safety enforcement by restricting desired controls to actions that
preserve a specified safe set~\cite{ames2019control}.
Higher-order CBFs (HOCBFs) extend this approach to requirements whose
dependence on control appears through higher derivatives, as when
acceleration regulates the separation between robots~\cite{xiao2019hocbf}.
Using robot-specific dynamics and admissible control sets, these
formulations can express safety requirements for systems with different
kinematic and actuation constraints~\cite{glotfelter2019hybrid}.
However, enforcing several pairwise barrier conditions requires a
common admissible control satisfying all of them.
Higher-order barrier inequalities can conflict with actuator limits~\cite{breeden2021high}, and local constraints constructed for separate encounters can be individually feasible yet jointly incompatible
~\cite{tan2022compatibility, isaly2024feasibility}.
Consequently, the feasibility of each pairwise condition does not
establish the feasibility of their combined safety controller.

Existing methods address this issue by adapting barrier parameters or
decay rates~\cite{xiao2021adaptive, liu2023auxiliary, zeng2021safety_decay},
relaxing or prioritizing selected constraints~\cite{lee2023hierarchical},
or allocating avoidance responsibility across agents
~\cite{lyu2023risk, lyu2022responsibility}.
Such mechanisms can enlarge the feasible control region or reduce
unnecessary demands on a constrained robot, but they do not by
themselves ensure that the resulting requirements admit one compatible
control across simultaneous encounters.
Joint or distributed optimization can coordinate these requirements
explicitly~\cite{tan2021distributed}, but a solution still requires a
feasible intersection of the imposed constraints and admissible
controls.  In particular, a robot may receive individually achievable
obligations that require conflicting actions when combined~\cite{jankovic2021liveness}.

{
An alternative is to certify a known safe continuation rather than
rely only on the instantaneous feasibility of barrier inequalities.
Backup-based methods identify states from which a prescribed controller
maintains safety and reaches an invariant terminal set
\cite{gurriet2018online}; backup CBFs preserve this property through
control constraints \cite{chen2021backup}. This principle has been
extended to decentralized collision avoidance under limited actuation
\cite{chen2020guaranteed_backup}.
For heterogeneous robots, however, the backup must remain admissible
under each robot's kinematic and actuation constraints; a generic
maneuver may be incompatible with restricted steering or forward-only
motion \cite{singletary2022safe}. Pairwise backup checks must also be
composed consistently across simultaneous encounters: certifying
different filtered actions for one robot does not certify the action it
executes. Predictive safety filters search more flexibly for safe
continuations \cite{wabersich2021predictive}, but their online
optimization scales with the prediction horizon and number of coupled
robots \cite{tearle2021predictive, park2020efficient}.
}

This raises the central question of our work: \textit{how can robots continue navigating while ensuring that an admissible recovery maneuver remains available throughout their interactions?}
A collision-free configuration may already be unsuitable for the prescribed backups if the robots are approaching too rapidly. The recovery condition must therefore account for motion along the backup trajectories rather than current separation alone.


To this end, we propose a Capability-Aware Backup-Recoverability Filter
(CA-BRF) that certifies model-specific backup maneuvers and consistently
filters desired controls to preserve their availability across
simultaneous interactions.
Our \textbf{main contributions} are:
1) \textit{Capability-dependent recovery certification.}
We develop a recovery certificate that evaluates whether heterogeneous
robots can execute their prescribed backup maneuvers while maintaining
pairwise safety until reaching rest.
By incorporating model-specific braking and turning responses into a
shared operational-space formulation, the certificate identifies a
backup-recoverable set determined by the robots' physical capabilities
and current motion;
2) \textit{Recoverability-preserving control selection.}
We design a safety filter that checks admissible interpolations between
desired and backup controls and selects the largest verified candidate
fraction that retains the recovery certificate.
The full backup serves as an explicit fallback when the candidate HOCBF optimization does not return a feasible action;
3) \textit{Consistent multi-robot filtering with recursive feasibility
and safety guarantees.}
We develop a component-wise selection mechanism that makes the
pairwise recovery certificates jointly valid under one executable
control per robot. Incorporating this consistency into the analysis, we establish
recursive feasibility of CA-BRF and safety for the discrete-time
multi-robot closed-loop system under the stated conditions.
\vspace{-10pt}

\section{Problem Formulation}
\label{sec:problem} \vspace{-3pt}




\subsection{Heterogeneous Robot Dynamics and Admissible Control}
\label{subsec:system}

Consider a team of $N$ heterogeneous mobile robots operating in a
common two-dimensional workspace, indexed by
$\mathcal I:=\{1,\ldots,N\}$.
Robot $i$ is associated with model $c_i\in\mathcal C$ and evolves as \vspace{-9pt}
\begin{equation}
    \dot{\mathbf x}_i
    =
    \mathbf f_{c_i}(\mathbf x_i)
    +
    \mathbf g_{c_i}(\mathbf x_i)\mathbf u_i,
    \quad
    \mathbf u_i\in\mathcal U_i(\mathbf x_i),
    \label{eq:robot_dynamics}
\end{equation}
where $\mathbf x_i$ and $\mathbf u_i$ denote its state and control
input.
The state-dependent set $\mathcal U_i(\mathbf x_i)$ contains the
controls that can be physically executed,
accounting for both actuation limits and model-specific kinematic
constraints.
We refer to the motion responses permitted by these dynamics and constraints as
the robot's \emph{motion capabilities}.

In this work, we consider five representative mobile-robot models: double
integrator (DI), unicycle (UNI), differential drive (DD), car-like
(CL), and forward-only car-like (FO)~\cite{lynch2017modern_robotics}.
For the nonholonomic models,
$\boldsymbol\nu_i=[v_i,\omega_i]^\top$ denotes the forward and angular
velocity, and $\mathbf u_i:=\dot{\boldsymbol\nu}_i$ is the
acceleration-level control.
Each model has a feasible velocity set $\mathcal V_{c_i}$
(Fig.~\ref{fig:control_sets}).
In particular, CL couples forward and angular motion through its
steering constraint, whereas FO additionally prohibits reverse motion.

Controls are updated at $t_k=k\Delta t$, with one-step state transition
$\mathbf x_i^{k+1}
    =
    F_i(\mathbf x_i^k,\mathbf u_i^k).$
Let $\mathcal A_i(\mathbf x_i)$ be the actuator-feasible acceleration
set. We define the admissible control set by \vspace{-9pt}
\begin{equation}
    \mathcal U_i(\mathbf x_i)
    :=
    \mathcal A_i(\mathbf x_i)
    \cap
    \left\{
        \mathbf u_i
        \;\middle|\;
        \boldsymbol\nu_i+\Delta t\,\mathbf u_i
        \in\mathcal V_{c_i}
    \right\},
    \label{eq:admissible_control_set}
\end{equation}
which enforces both the current acceleration limit and the
next-step velocity constraint.
For DI, $\boldsymbol\nu_i=[v_{x,i},v_{y,i}]^\top$ denotes the Cartesian
velocity.

Let $\mathcal X_i^{\rm adm}$ denote the corresponding admissible state
domain, and assume that
$F_i(\mathbf x_i,\mathbf u_i)\in\mathcal X_i^{\rm adm}$ for all
$\mathbf x_i\in\mathcal X_i^{\rm adm}$ and
$\mathbf u_i\in\mathcal U_i(\mathbf x_i)$.
This condition ensures that an admissible control does not drive the
robot outside its model-specific kinematic domain. \vspace{-6pt}

\subsection{Shared Operational Space and Pairwise Safety}
\label{subsec:operational_safety}

Although the robots have different state and control representations, collision avoidance requires a common geometric description.
For each nonholonomic robot, we define an offset operational point~\cite{lynch2017modern_robotics}
\begin{equation}
    \mathbf p_i
    =
    \begin{bmatrix}
        x_i+x_{r,c_i}\cos\phi_i\\
        y_i+x_{r,c_i}\sin\phi_i
    \end{bmatrix},
    \label{eq:operational_point}
\end{equation}
where $x_{r,c_i}>0$ is the model-dependent offset from the body
position along heading $\phi_i$.
Its operational velocity is
\begin{equation}
    \dot{\mathbf p}_i
    =
    \mathbf J_i(\phi_i)\boldsymbol\nu_i,
    \quad
    \mathbf J_i(\phi_i)
    =
    \begin{bmatrix}
        \cos\phi_i &
        -x_{r,c_i}\sin\phi_i\\
        \sin\phi_i &
         x_{r,c_i}\cos\phi_i
    \end{bmatrix}.
    \label{eq:operational_velocity}
\end{equation}
Differentiating gives \vspace{-3pt}
\begin{equation}
    \ddot{\mathbf p}_i
    =
    \boldsymbol\eta_i(\mathbf x_i)
    +
    \mathbf G_i(\mathbf x_i)\mathbf u_i,
    \label{eq:operational_dynamics}
\end{equation}
where $\boldsymbol\eta_i=\dot{\mathbf J}_i\boldsymbol\nu_i$ and
$\mathbf G_i=\mathbf J_i$.
For DI, we use its Cartesian position directly, giving
$x_{r,c_i}=0$, $\boldsymbol\eta_i=\mathbf0$, and
$\mathbf G_i=\mathbf I_2$.
Hence, the heterogeneous robots share the operational form
\eqref{eq:operational_dynamics}, while their different motion
capabilities remain encoded through
$\boldsymbol\eta_i$, $\mathbf G_i$, and
$\mathcal U_i(\mathbf x_i)$.

We next define safety in this shared space.
Let $\mathbf c_i\in\mathbb R^2$ be the body-center position of robot
$i$, and let $r_i^{\rm phys}$ be the maximum distance from
$\mathbf c_i$ to its physical body.
Since $\|\mathbf p_i-\mathbf c_i\|=x_{r,c_i}$, define
$r_i:=r_i^{\rm phys}+x_{r,c_i}$ and $R_{ij}:=r_i+r_j$.
By the triangle inequality,
$\|\mathbf c_i-\mathbf c_j\|
\geq\|\mathbf p_i-\mathbf p_j\|-x_{r,c_i}-x_{r,c_j}$;
therefore, $\|\mathbf p_i-\mathbf p_j\|\geq R_{ij}$ is sufficient to
prevent overlap of the physical robot bodies.
Let $\mathbf x:=[\mathbf x_1^\top,\ldots,\mathbf x_N^\top]^\top$
denote the joint state and
$\Delta\mathbf p_{ij}:=\mathbf p_i-\mathbf p_j$. We define the pairwise safety function and safe set as
\begin{equation}
    h_{ij}(\mathbf x)
    :=
    \|\Delta\mathbf p_{ij}\|^2-R_{ij}^2,
    \quad
    \mathcal C_{ij}
    :=
    \{\mathbf x\mid h_{ij}(\mathbf x)\geq0\}.
    \label{eq:pairwise_safety_function}
\end{equation}
The safety objective is to maintain
$h_{ij}(\mathbf x^k)\geq0$ for all $i<j$ and all control updates $k$.

\subsection{HOCBF-Based Candidate Controller}
\label{subsec:hocbf_nominal_controller}

We first construct a decentralized HOCBF-based candidate controller
that provides the desired navigation action for CA-BRF.
Because $h_{ij}$ depends on position while
$\mathbf u_i$ and $\mathbf u_j$ regulate acceleration through
\eqref{eq:operational_dynamics}, $h_{ij}$ has relative degree two.
Let $\Delta\dot{\mathbf p}_{ij}
:=\dot{\mathbf p}_i-\dot{\mathbf p}_j$, so that
$\dot h_{ij}
=2\Delta\mathbf p_{ij}^{\top}\Delta\dot{\mathbf p}_{ij}$.
For gains $\lambda_1,\lambda_2>0$, define the first HOCBF layer
$\psi_{ij}^{H}:=\dot h_{ij}+\lambda_1h_{ij}$.
Enforcing
$\dot\psi_{ij}^{H}+\lambda_2\psi_{ij}^{H}\geq0$ gives
\begin{equation}
    \ddot h_{ij}
    +
    (\lambda_1+\lambda_2)\dot h_{ij}
    +
    \lambda_1\lambda_2h_{ij}
    \geq0.
    \label{eq:hocbf_condition}
\end{equation}

Substituting~\eqref{eq:operational_dynamics} into
\eqref{eq:hocbf_condition} yields the coupled pairwise condition
\begin{equation}
    2\Delta\mathbf p_{ij}^{\top}
    \left(
        \mathbf G_i\mathbf u_i-\mathbf G_j\mathbf u_j
    \right)
    +
    2\Delta\mathbf p_{ij}^{\top}
    \left(
        \boldsymbol\eta_i-\boldsymbol\eta_j
    \right)
    +
    \Upsilon_{ij}
    \geq0,
    \label{eq:joint_hocbf_condition}
\end{equation}
where
$\Upsilon_{ij}:=
2\|\Delta\dot{\mathbf p}_{ij}\|^2
+(\lambda_1+\lambda_2)\dot h_{ij}
+\lambda_1\lambda_2h_{ij}$.
To decentralize this condition, we assign equal responsibility to the
two robots. Robot $i$ enforces
$2\Delta\mathbf p_{ij}^{\top}\mathbf G_i\mathbf u_i
    \geq
    -2\Delta\mathbf p_{ij}^{\top}\boldsymbol\eta_i
    -\frac{1}{2}\Upsilon_{ij},$
while robot $j$ applies the symmetric condition. Adding the two local
conditions recovers~\eqref{eq:joint_hocbf_condition}.

Let
$\mathbf a_{ij}^{\top}:=
2\Delta\mathbf p_{ij}^{\top}\mathbf G_i$ and
$b_{ij}:=
-2\Delta\mathbf p_{ij}^{\top}\boldsymbol\eta_i
-\frac{1}{2}\Upsilon_{ij}$.
Given a nominal navigation command $\mathbf u_i^{\rm nom}$ and
neighbor set $\mathcal N_i$, robot $i$ computes
\begin{equation}
\begin{aligned}
    \widetilde{\mathbf u}_i^{\rm cand}
    \in
    \argmin_{\mathbf u_i}\quad&
    \|\mathbf u_i-\mathbf u_i^{\rm nom}\|_2^2\\
    \mathrm{s.t.}\quad&
    \mathbf u_i\in\mathcal U_i(\mathbf x_i),\\
    &
    \mathbf a_{ij}^{\top}\mathbf u_i
    \geq b_{ij},
    \quad \forall j\in\mathcal N_i .
\end{aligned}
\label{eq:hocbf_candidate_qp}
\end{equation}
For a nonconvex admissible set, the problem is solved over its convex
branches and the feasible solution with the smallest objective is
selected.
When the QP does not return a solution, a braking-based fallback
$\mathbf u_i^{\rm fb}$ is used to keep the candidate controller
defined.
In either case, the resulting raw command is projected onto the exact
admissible set:
\begin{equation}
    \mathbf u_i^{\rm cand}
    :=
    \Pi_{\mathcal U_i(\mathbf x_i)}
    \left(
        \widehat{\mathbf u}_i^{\rm cand}
    \right),
    \label{eq:candidate_projection}
\end{equation}
where \vspace{-3pt}
\begin{equation}
    \widehat{\mathbf u}_i^{\rm cand}
    :=
    \begin{cases}
        \widetilde{\mathbf u}_i^{\rm cand},
        & \text{if \eqref{eq:hocbf_candidate_qp} returns a solution},\\
        \mathbf u_i^{\rm fb},
        & \text{otherwise}.
    \end{cases}
    \label{eq:raw_candidate}
\end{equation}
Hence,
$\mathbf u_i^{\rm cand}\in\mathcal U_i(\mathbf x_i)$.

Importantly, $\mathbf u_i^{\rm cand}$ is a desired navigation
action rather than the final safety certificate.
Multiple local HOCBF rows can be mutually incompatible, and neither
the fallback nor the final projection is required to preserve every
row in \eqref{eq:hocbf_candidate_qp}.
The final safety guarantee developed below therefore does not require
the candidate HOCBF-QP to remain feasible.

\subsection{Filtering Objectives}
\label{subsec:filtering_objectives}

The preceding controller specifies what each robot would prefer to
execute, but it does not resolve the feasibility loss caused by
simultaneous interactions.
We therefore seek an applied control $\mathbf u^{\rm app} = \mathcal F \left(\mathbf x, \mathbf u^{\rm cand}\right),$ 
where
$\mathbf u^{\rm cand}
=
(\mathbf u_1^{\rm cand},\ldots,\mathbf u_N^{\rm cand})$.
The filter is designed around three requirements:
(i) Recovery must be evaluated using backup maneuvers that respect
each robot's model- and state-dependent admissible motion.
(ii) An accepted navigation action must leave the successor state recoverable, rather than merely being collision-free at the current step.
(iii) When a robot participates in several interactions, all
pairwise recovery conditions must be certified under the same control
that the robot actually executes.
The next section develops CA-BRF to satisfy these requirements.

\section{Capability-Aware Backup-Recoverability Filter}
\label{sec:cabrf}


In this section, we introduce the \textbf{Capability-Aware Backup-Recoverability Filter (CA-BRF)}, a decentralized filter that certifies recoverability under model-specific backups and consistently
filters desired controls to preserve it across multi-robot interactions.
Our approach is guided by a simple principle: an action is
accepted only when it preserves access to an executable recovery
maneuver. CA-BRF thus connects the desired HOCBF candidate with a
model-specific backup and determines how much of the candidate can be
executed without losing recoverability. Algorithm~\ref{alg:cabrf} summarizes this procedure. \vspace{-3pt}


\begin{algorithm}[t]\footnotesize
\caption{Capability-Aware Backup-Recoverability Filter (CA-BRF)} 
\label{alg:cabrf}
\begin{algorithmic}[1]
\REQUIRE States $\{\mathbf x_i\}$, nominal controls
$\{\mathbf u_i^{\rm nom}\}$
\ENSURE Applied controls $\{\mathbf u_i^{\rm app}\}$

\FOR{each robot $i$}
    \STATE Compute $\mathbf u_i^{\rm cand}$ using
    \eqref{eq:hocbf_candidate_qp}--\eqref{eq:candidate_projection}
    \algcomment{desired candidate}
    \STATE Compute
    $\mathbf u_i^b=\pi_i^b(\mathbf x_i)$
    \algcomment{admissible backup}
\ENDFOR

\STATE Initialize
$\mathbf u_i^{\rm app}\leftarrow\mathbf u_i^{\rm cand}$,
$\forall i$
\STATE Build active graph $\mathcal G$ and its connected components
\STATE
$\Lambda=\{1=\lambda_0>\cdots>\lambda_M=0\}$
\algcomment{includes full backup}

\FOR{each connected component $\mathcal K$ with
      $\mathcal E_{\mathcal K}\neq\varnothing$}
    \FOR{$\lambda\in\Lambda$ in descending order}
        \STATE
        $\mathbf u_i(\lambda)
        \leftarrow
        \mathbf u_i^b+
        \lambda(\mathbf u_i^{\rm cand}-\mathbf u_i^b)$,
        $\forall i\in\mathcal K$
        \algcomment{interpolate}

        \STATE Check
        $\mathcal M_{\mathcal K}(\lambda)$
        \algcomment{admissibility + recoverability}

        \IF{$\mathcal M_{\mathcal K}(\lambda)=1$}
            \STATE $\lambda_{\mathcal K}\leftarrow\lambda$
            \algcomment{largest verified fraction}
            \STATE \textbf{break}
        \ENDIF
    \ENDFOR

    \STATE
    $\mathbf u_i^{\rm app}
    \leftarrow
    \mathbf u_i(\lambda_{\mathcal K})$,
    $\forall i\in\mathcal K$
\ENDFOR
\RETURN $\{\mathbf u_i^{\rm app}\}$ 
\end{algorithmic} 
\end{algorithm}

\subsection{Capability-Compatible Backup Maneuvers}
\label{subsec:backup}

To obtain a recovery action whose availability does not depend on the
candidate QP, we assign each robot an admissible state-feedback backup
policy.
The backup is chosen to bring the robot toward rest, so that once all
interacting robots have stopped, their relative configuration remains
unchanged.

For robot $i$, let
$\mathbf u_i^b=\pi_i^b(\mathbf x_i)$, where
$\pi_i^b(\mathbf x_i)\in\mathcal U_i(\mathbf x_i)$.
%
This backup is capability-dependent because heterogeneous robots
cannot all stop through the same control response.
For DI and UNI, the corresponding velocity components can be reduced
directly toward zero.
For DD, braking must remain consistent with its wheel constraints.
For CL and FO, the stopping maneuver must additionally remain inside
the speed-dependent steering envelope, and FO must maintain
nonnegative forward velocity.
Thus, the resulting braking and turning response depends on the robot
model, admissible control limits, and current state.

Starting from $\mathbf x_i$, repeated application of this policy
produces the backup rollout
\begin{equation}
\begin{aligned}
    \boldsymbol\phi_{i,0}^b(\mathbf x_i)
    &:=
    \mathbf x_i,\\
    \boldsymbol\phi_{i,k+1}^b(\mathbf x_i)
    &:=
    F_i
    \left(
        \boldsymbol\phi_{i,k}^b(\mathbf x_i),
        \pi_i^b
        \left(
            \boldsymbol\phi_{i,k}^b(\mathbf x_i)
        \right)
    \right).
\end{aligned}
\label{eq:backup_rollout}
\end{equation}
We assume that the backup remains admissible along this rollout and
reaches a stationary state in finite time. For a prescribed backup
policy, a finite stopping-time bound can be obtained from the
model-specific actuation limits and control period.

Let $K_i^{\rm stop}(\mathbf x_i)$ denote the number of steps required
for robot $i$ to reach rest, and assume the uniform bound
$K^{\rm stop}:=\sup_{i,\,\mathbf x_i\in\mathcal X_i^{\rm adm}}
K_i^{\rm stop}(\mathbf x_i)<\infty$.
We select $H\geq K^{\rm stop}+1$ so that the complete stopping
maneuver and the resulting stationary state are included in the
finite rollout.

Admissibility of the backup alone is not sufficient for pairwise
safety.
Two robots may each brake within their physical limits and still
violate their required separation before coming to rest.
We therefore evaluate the two backup trajectories jointly and certify
the states from which these prescribed recovery maneuvers remain safe.

\subsection{Backup-Recoverability Certification}
\label{subsec:backup_recoverability}

For pair $(i,j)$, define the joint backup state at rollout step $k$ as 
$\boldsymbol\Phi_{ij,k}^b(\mathbf x)
    :=
    \left(
        \boldsymbol\phi_{i,k}^b(\mathbf x_i),
        \boldsymbol\phi_{j,k}^b(\mathbf x_j)
    \right).$
%
%
Geometric separation alone indicates whether the robots overlap, but
does not characterize how rapidly the separation margin is being
lost.
We therefore evaluate both $h_{ij}$ and the first-layer barrier
quantity $\psi_{ij}^{B}:=\dot h_{ij}+\lambda_B h_{ij}$, where
$\lambda_B>0$.
The first quantity monitors physical separation, whereas
$\psi_{ij}^{B}$ additionally accounts for the pair's closing motion.

We summarize the complete stopping trajectory by the minimum margins \vspace{-2pt}
\begin{equation}
    B_{ij}^{h}(\mathbf x)
    :=
    \min_{0\leq k\leq H-1}
    h_{ij}
    \left(
        \boldsymbol\Phi_{ij,k}^b(\mathbf x)
    \right),
    \label{eq:backup_h_margin}
\end{equation}
and
\begin{equation}
    B_{ij}^{\psi}(\mathbf x)
    :=
    \min_{0\leq k\leq H-1}
    \psi_{ij}^{B}
    \left(
        \boldsymbol\Phi_{ij,k}^b(\mathbf x)
    \right).
    \label{eq:backup_psi_margin}
\end{equation}
Hence,
$B_{ij}^{h}\geq0$ requires the prescribed separation to hold
throughout recovery, while
$B_{ij}^{\psi}\geq0$ requires the first-layer barrier condition to
remain valid along the same trajectory.

We define the corresponding recoverable set as follows.

\begin{definition}[Backup-Recoverable Set]
\label{def:backup_recoverable_set}
The backup-recoverable set for pair $(i,j)$ is
\begin{equation}
    \mathcal C_{ij}^{B}
    :=
    \left\{
        \mathbf x
        \;\middle|\;
        B_{ij}^{h}(\mathbf x)\geq0,\;
        B_{ij}^{\psi}(\mathbf x)\geq0
    \right\}.
    \label{eq:backup_recoverable_set}
\end{equation}
\end{definition}

Membership in $\mathcal C_{ij}^{B}$ has a direct interpretation:
starting from the current state, both robots can initiate their
prescribed admissible backup responses and maintain the certified
pairwise safety conditions until they stop.
Because the rollout
$\boldsymbol\Phi_{ij,k}^{b}$
is generated using each robot's actual model- and state-dependent
backup response, this certificate is capability-dependent.
Robots in the same relative geometry can therefore have different
recoverability margins when their braking or turning capabilities
differ.
The certificate is also policy-dependent: a state outside
$\mathcal C_{ij}^{B}$ is not certified by the prescribed backups,
although another recovery maneuver may exist.

A central property of this set is that the certified backup preserves
its own recoverability.
Let
$F_{ij}^{b}(\mathbf x)
    :=
    \left(
        F_i(\mathbf x_i,\pi_i^b(\mathbf x_i)),
        F_j(\mathbf x_j,\pi_j^b(\mathbf x_j))
    \right)$
%
%
%
denote one joint backup step.
Because the subsequent rollout is the original rollout shifted by one
step and the terminal state is stationary,
\begin{equation}
    B_{ij}^{h}
    \left(
        F_{ij}^{b}(\mathbf x)
    \right)
    \geq
    B_{ij}^{h}(\mathbf x),
    \quad
    B_{ij}^{\psi}
    \left(
        F_{ij}^{b}(\mathbf x)
    \right)
    \geq
    B_{ij}^{\psi}(\mathbf x).
    \label{eq:backup_margin_shift}
\end{equation}

This gives the following result.
\begin{lemma}[Recursive Backup Recoverability]
\label{lem:backup_preservation}
If $\mathbf x\in\mathcal C_{ij}^{B}$, then
$F_{ij}^{b}(\mathbf x)\in\mathcal C_{ij}^{B}$.
\end{lemma}

\begin{proof}
After one joint backup step, every nonterminal state of the new
backup rollout is a state already contained in the previous rollout.
Since $H>K^{\rm stop}$, the additional terminal state is stationary
and therefore has the same values of
$h_{ij}$ and $\psi_{ij}^{B}$ as the preceding stationary state.
Consequently, neither minimum in
\eqref{eq:backup_h_margin}--\eqref{eq:backup_psi_margin}
can decrease, which gives
\eqref{eq:backup_margin_shift} and proves the result.
\end{proof}

Lemma~\ref{lem:backup_preservation} provides CA-BRF with a known
recoverable endpoint.
Rather than applying this endpoint at every step, we next use it to
retain as much of the desired candidate as possible.

\subsection{Consistent Recoverability-Preserving Filtering}
\label{subsec:recoverability_filter}


{The backup policies provide a known recovery option, but applying them
throughout an interaction would unnecessarily sacrifice navigation
performance. Given the candidate controls
$\mathbf u^{\rm cand}
=(\mathbf u_1^{\rm cand},\ldots,\mathbf u_N^{\rm cand})$
from Section~\ref{subsec:hocbf_nominal_controller}, CA-BRF therefore preserves
the backup as a fallback while retaining as much of the desired action as
recoverability permits. This tradeoff must be resolved consistently across
simultaneous interactions. Because a robot shared by several interacting pairs
can execute only one control, all associated recovery conditions must be
satisfied under that same action. We capture this coupling through connected
interaction components and perform the following selection for each component.}

At each control step, we represent the interactions requiring
recoverability certification by the graph
$\mathcal G_k=(\mathcal I,\mathcal E_k)$.
An edge $(i,j)$, with $i<j$, is active when
$\|\mathbf p_i^k-\mathbf p_j^k\|-R_{ij}<d_{\rm act}$,
where $d_{\rm act}>0$ is the activation range.
Robots connected through active edges are coupled by their pairwise
recovery conditions.
Let $\mathcal K$ denote a connected component of $\mathcal G_k$ and
$\mathcal E_{\mathcal K}$ its active edges.
Processing each component separately couples all decisions that share
a robot while leaving independent interactions without distraction.

For each component, CA-BRF interpolates between the candidate and
backup controls,
\begin{equation}
    \mathbf u_i(\lambda)
    :=
    \mathbf u_i^b
    +
    \lambda
    \left(
        \mathbf u_i^{\rm cand}-\mathbf u_i^b
    \right),
    \quad
    i\in\mathcal K,
    \label{eq:control_interpolation}
\end{equation}
where $\lambda=1$ gives the full candidate and $\lambda=0$ gives the
full backup.
The search is restricted to the finite ordered set
$\Lambda
    :=
    \{1=\lambda_0>\lambda_1>\cdots>\lambda_M=0\}.$
%
%
Thus, the common scalar $\lambda$ coordinates the amount of
intervention within a connected interaction component, while the
endpoints
$\mathbf u_i^{\rm cand}$ and $\mathbf u_i^b$
remain robot-specific.

For a candidate fraction $\lambda$, define the one-step successor
$\mathbf x_i^{+}(\lambda)
    :=
    F_i
    \left(
        \mathbf x_i,
        \mathbf u_i(\lambda)
    \right).$
%
%
CA-BRF then performs two checks.
First, the interpolated control must remain executable:
\begin{equation}
    A_i(\lambda)
    :=
    \mathbf 1
    \left[
        \mathbf u_i(\lambda)
        \in
        \mathcal U_i(\mathbf x_i)
    \right].
    \label{eq:control_admissibility_indicator}
\end{equation}
This explicit check is necessary because
$\mathcal U_i(\mathbf x_i)$ need not be convex; two admissible
endpoints do not in general imply that every interpolation between
them is admissible.

Second, the successor must preserve the certified recovery:
\begin{equation}
    P_{ij}(\lambda)
    :=
    \mathbf 1
    \left[
        \left(
            \mathbf x_i^{+}(\lambda),
            \mathbf x_j^{+}(\lambda)
        \right)
        \in
        \mathcal C_{ij}^{B}
    \right],
    \quad
    (i,j)\in\mathcal E_{\mathcal K}.
    \label{eq:recoverability_indicator}
\end{equation}
The successor check is essential.
A pair may be recoverable at the current state while an aggressive
candidate moves it to a state from which the prescribed backups can
no longer maintain safety.

A candidate fraction is accepted for component $\mathcal K$ only
when both conditions hold simultaneously:
\begin{equation}
    \mathcal M_{\mathcal K}(\lambda)
    :=
    \bigwedge_{i\in\mathcal K}
    A_i(\lambda)
    \;\wedge\;
    \bigwedge_{(i,j)\in\mathcal E_{\mathcal K}}
    P_{ij}(\lambda).
    \label{eq:component_mask}
\end{equation}
This is the verification step in
Line~11 of Algorithm~\ref{alg:cabrf}.
The common component-wise decision is important when one robot belongs
to several pairwise interactions: every certificate is evaluated
under the same control that the shared robot will actually execute.

CA-BRF selects
\begin{equation}
    \lambda_{\mathcal K}^{*}
    :=
    \max
    \left\{
        \lambda\in\Lambda
        \;\middle|\;
        \mathcal M_{\mathcal K}(\lambda)=1
    \right\},
    \label{eq:selected_lambda}
\end{equation}
and applies
$\mathbf u_i^{\rm app} = \mathbf u_i^b+\lambda_{\mathcal K}^{*} \left(\mathbf u_i^{\rm cand}-\mathbf u_i^b\right)$ with $i\in\mathcal K.$
%
%
The grid is evaluated explicitly in descending order, as shown in
Lines~9--16 of Algorithm~\ref{alg:cabrf}, and we make no monotonicity
assumption on feasibility with respect to $\lambda$.
Therefore,
$\lambda_{\mathcal K}^{*}$ is the largest verified candidate
fraction, not the result of a monotone line search.
A robot with no active edge directly executes
$\mathbf u_i^{\rm cand}$.

The inclusion of $\lambda=0$ in $\Lambda$ is particularly important.
At this endpoint, $\mathbf u_i(0)=\mathbf u_i^b$ for every robot in
the component. If the component is backup-recoverable,
Lemma~\ref{lem:backup_preservation} guarantees that the full backup
preserves every pairwise recoverability certificate. Moreover,
$\mathbf u_i(0)=\mathbf u_i^b\in\mathcal U_i(\mathbf x_i)$, so
$A_i(0)=1$ for all $i\in\mathcal K$. Thus, $\lambda=0$ provides an
explicit feasible endpoint independently of the candidate HOCBF-QP. \vspace{-5pt}

\subsection{Recursive Feasibility and Multi-Robot Safety}
\label{subsec:safety_guarantee}

We now analyze the system obtained by repeatedly applying the
component-wise decision above. \vspace{-2pt}

\begin{theorem}[Recursive Feasibility and Safety of CA-BRF]
\label{thm:safety}
Assume that
(i) all robots start from admissible states,
(ii) every initially active pair belongs to
$\mathcal C_{ij}^{B}$, and
(iii) every newly active pair belongs to
$\mathcal C_{ij}^{B}$ when it enters the active graph.
Then, at every control step, CA-BRF admits at least one feasible
$\lambda\in\Lambda$ for each active connected component.
Every applied control satisfies
$\mathbf u_i^{\rm app} \in \mathcal U_i(\mathbf x_i),$
and every active pair remains in
$\mathcal C_{ij}^{B}$.
Consequently, at every control update time $t_k=k\Delta t$, \vspace{-3pt}
\begin{equation}
    h_{ij}(\mathbf x^k)\geq0,
    \quad
    \forall i<j,\quad
    \forall k,
    \label{eq:sampled_safety}
\end{equation}
and the physical robot bodies remain non-overlapping by the
construction in Sec.~\ref{subsec:operational_safety}.
\end{theorem}

\begin{proof}
The grid $\Lambda$ contains $\lambda=0$. For this value,
$\mathbf u_i(0)=\mathbf u_i^b\in\mathcal U_i(\mathbf x_i)$ for every
robot in an active component. If an active pair belongs to
$\mathcal C_{ij}^{B}$, Lemma~\ref{lem:backup_preservation} guarantees
that one joint backup step leaves the pair in
$\mathcal C_{ij}^{B}$. Hence,
$\mathcal M_{\mathcal K}(0)=1$ for every backup-recoverable component,
and the feasible set in \eqref{eq:selected_lambda} is nonempty.

The selected $\lambda_{\mathcal K}^{*}$ passes every admissibility and
successor-recoverability check in \eqref{eq:component_mask}.
Therefore, applied controls in active components are admissible and
every active pair remains in $\mathcal C_{ij}^{B}$ after the control
update. Robots with no active edge directly execute
$\mathbf u_i^{\rm cand}\in\mathcal U_i(\mathbf x_i)$.
Newly active pairs satisfy the same recoverability property by
assumption. Repeating this argument at successive control updates
establishes recursive filter feasibility and preservation of the
backup-recoverable sets.

Finally, the backup rollout in \eqref{eq:backup_h_margin} includes the
current state at $k=0$. Thus,
$\mathbf x\in\mathcal C_{ij}^{B}$ implies
$h_{ij}(\mathbf x)\geq0$. An inactive pair satisfies
$\|\mathbf p_i-\mathbf p_j\|-R_{ij}\geq d_{\rm act}>0$ and is also
safe. Hence, \eqref{eq:sampled_safety} holds for every pair at every
control update.
\end{proof}
\vspace{-5pt}

Theorem~\ref{thm:safety} establishes recursive feasibility of the
\emph{CA-BRF filter}, not of the candidate HOCBF-QP, and guarantees
safety at each control-update instant. The HOCBF-QP generates the
desired candidate control, whereas CA-BRF certifies the final applied
control and provides the recursive safety guarantee.
Finally, the component-wise check admits a distributed realization.
Each robot can evaluate the admissibility vector
$\{A_i(\lambda)\}_{\lambda\in\Lambda}$, while each active pair
evaluates its recoverability vector
$\{P_{ij}(\lambda)\}_{\lambda\in\Lambda}$.
Combining these binary vectors over a connected component yields
$\mathcal M_{\mathcal K}(\lambda)$ and therefore the common
$\lambda_{\mathcal K}^{*}$ used by all robots in that component. \vspace{-3pt}

\section{Experiments}
\label{sec:experiments}
\vspace{-0.3em}


\begin{table*}[!t]
\centering

\begin{minipage}[t]{0.72\textwidth}
\centering

\begin{minipage}[t][2\baselineskip][t]{\linewidth}
\captionof{table}{
Heterogeneous multi-robot random start--goal navigation across team sizes
($N\in\{10,20,30\}$, 50 trials each).
}
\label{tab:main_results}
\end{minipage}

\footnotesize
\setlength{\tabcolsep}{2.2pt}
\renewcommand{\arraystretch}{1.08}

\begin{tabular*}{\linewidth}{
    @{\extracolsep{\fill}}
    l cccc cccc cccc
    @{}
}
\toprule
& \multicolumn{4}{c}{$N=10$}
& \multicolumn{4}{c}{$N=20$}
& \multicolumn{4}{c}{$N=30$} \\
\cmidrule(lr){2-5}
\cmidrule(lr){6-9}
\cmidrule(lr){10-13}

\textbf{Method}
& \texttt{SR}$\uparrow$
& \texttt{AR}$\uparrow$
& \texttt{SVR}$\downarrow$
& \texttt{Pen.}$\downarrow$
& \texttt{SR}$\uparrow$
& \texttt{AR}$\uparrow$
& \texttt{SVR}$\downarrow$
& \texttt{Pen.}$\downarrow$
& \texttt{SR}$\uparrow$
& \texttt{AR}$\uparrow$
& \texttt{SVR}$\downarrow$
& \texttt{Pen.}$\downarrow$
\\
\midrule

\textbf{CA-BRF (Ours)}
& \cellcolor{gray!25}\textbf{98.0}
& \cellcolor{gray!10}98.0
& \cellcolor{gray!25}\textbf{0.0}
& \cellcolor{gray!25}\textbf{0.000}
& \cellcolor{gray!25}\textbf{92.0}
& \cellcolor{gray!10}92.0
& \cellcolor{gray!25}\textbf{0.0}
& \cellcolor{gray!25}\textbf{0.000}
& \cellcolor{gray!25}\textbf{85.6}
& \cellcolor{gray!10}85.6
& \cellcolor{gray!25}\textbf{0.0}
& \cellcolor{gray!25}\textbf{0.000}
\\

BRF-Shared
& 96.8 & 97.0
& \cellcolor{gray!10}2.0
& \cellcolor{gray!10}0.001
& \cellcolor{gray!10}88.8
& 89.1
& \cellcolor{gray!10}12.0
& \cellcolor{gray!10}0.005
& 74.9 & 75.8
& \cellcolor{gray!10}28.0
& \cellcolor{gray!10}0.004
\\

HOCBF (w/o BRF)
& \cellcolor{gray!10}97.6
& \cellcolor{gray!10}98.0
& \cellcolor{gray!10}2.0
& 0.058
& 88.4
& \cellcolor{gray!25}\textbf{94.2}
& 52.0
& 0.055
& \cellcolor{gray!10}79.5
& \cellcolor{gray!25}\textbf{90.2}
& 86.0
& 0.072
\\

Backup-CBF~\cite{chen2020guaranteed_backup}
& \cellcolor{gray!10}97.6
& \cellcolor{gray!25}\textbf{98.6}
& 6.0 & 0.023
& 83.2 & 86.1 & 44.0 & 0.022
& 58.1 & 60.5 & 84.0 & 0.029
\\

sRCBF~\cite{kim2025srcbf}
& 54.8 & 55.8 & 10.0 & 0.332
& 19.4 & 21.7 & 90.0 & 0.428
& 5.1 & 6.7 & 100.0 & 0.724
\\
\bottomrule
\end{tabular*}
\end{minipage}
\hfill
\begin{minipage}[t]{0.27\textwidth}
\centering

\begin{minipage}[t][3.6\baselineskip][t]{\linewidth}
\captionof{table}{
Mean number of pairwise safety-violation events per trial.
}
\label{tab:violation_events}
\end{minipage}

\footnotesize
\setlength{\tabcolsep}{1.5pt}
\renewcommand{\arraystretch}{1.08}

\begin{tabular*}{\linewidth}{
    @{\extracolsep{\fill}}
    l ccc
    @{}
}
\toprule
\textbf{Method}
& $N=10$
& $N=20$
& $N=30$
\\
\midrule

\textbf{CA-BRF}
& \textbf{0.00}
& \textbf{0.00}
& \textbf{0.00}
\\

BRF-Shared
& 0.02 & 0.12 & 0.30
\\

HOCBF
& 0.02 & 1.04 & 3.70
\\

Backup-CBF
& 0.06 & 0.70 & 1.70
\\

sRCBF
& 0.10 & 2.00 & 8.70
\\
\bottomrule
\end{tabular*}
\end{minipage}

\vspace{3pt}

\begin{minipage}{\textwidth}
\footnotesize
\texttt{SR}: safe success rate;
\texttt{AR}: arrival rate;
\texttt{SVR}: percentage of trials with at least one safety-set violation;
\texttt{Pen.}: mean maximum penetration over trials with safety violations
(zero when no violation occurs).
\textbf{Bold} with \colorbox{gray!25}{dark shading} indicates the best result;
\colorbox{gray!10}{light shading} indicates the second best.
\end{minipage}

\vspace{-1.2em}
\end{table*}

In this section, we evaluate the safety and navigation performance of
CA-BRF under increasingly dense heterogeneous multi-robot interactions
and analyze the contributions of capability-aware recovery evaluation,
recovery certification, and component-wise consistent filtering.
Specifically, we evaluate CA-BRF through
(i) random start--goal navigation with up to $N=30$ heterogeneous
robots against representative safety controllers,
(ii) targeted ablation studies that isolate these three design
components, and
(iii) a five-robot physical experiment involving all five kinematic
classes.

\begin{figure}[t]
    \captionsetup{skip=2pt}
    \centering
    \includegraphics[width=0.48\textwidth]
    {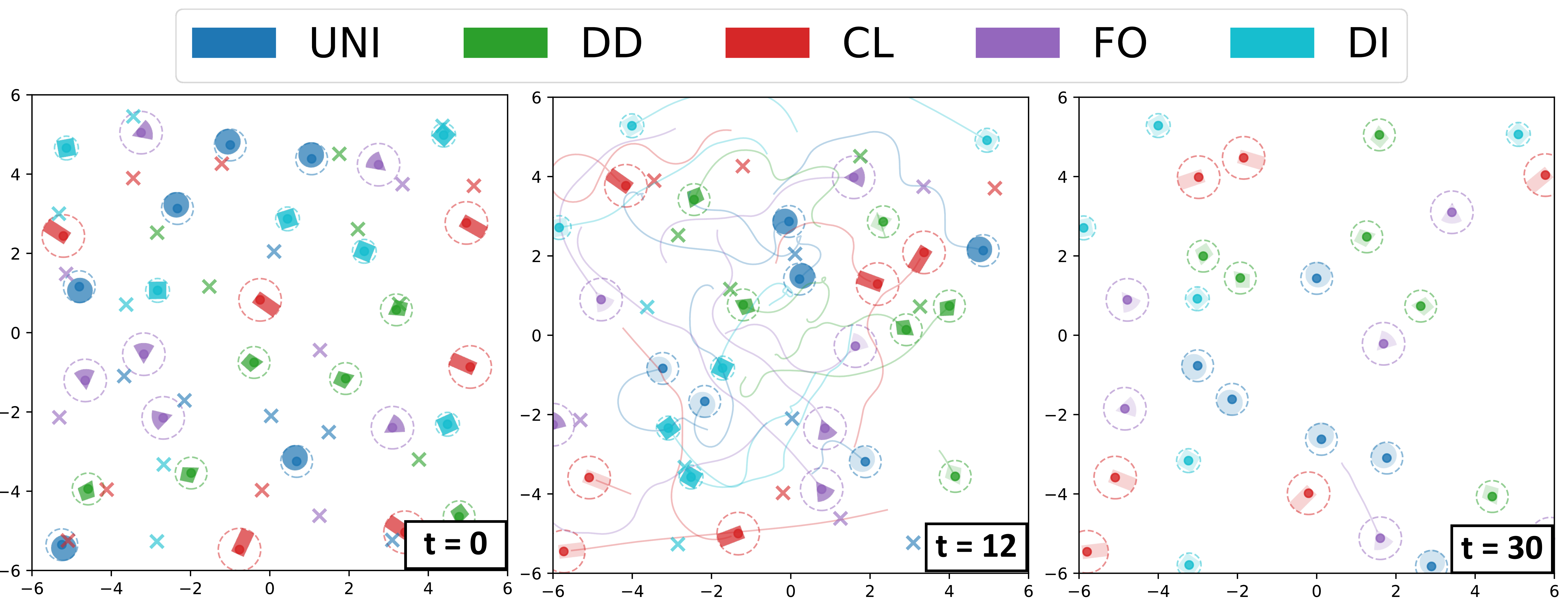}
    \caption{Random start--goal navigation with $N=30$ heterogeneous robots
(six per model) at multiple timestamps. Robots start from randomly
sampled positions, and $\times$ markers denote their assigned goals.
UNI, DI, DD, CL, and FO robots are shown as blue circles, light-blue
circles, green diamonds, red rectangles, and purple sectors,
respectively. Dashed outlines show the operational-space safety
radii $r_i$ centered at the corresponding operational points
(Sec.~\ref{sec:problem}).
    }
    \label{fig:random_start_goal}
    \vspace{-1.5em}
\end{figure}

\subsubsection*{\textbf{Experiment Settings}}


We consider random start--goal navigation in an obstacle-free
two-dimensional plane.
Initial and goal positions are sampled from an
$11\times11\,\mathrm{m}^2$ region, with minimum start--goal and
initial inter-robot distances of $4.0$~m and $1.5$~m, respectively
(Fig.~\ref{fig:random_start_goal}).
We evaluate teams of $N\in\{10,20,30\}$ robots, evenly distributed
across the five kinematic classes, over $50$ random trials per team
size.\footnote{Each trial runs for at most $1000$ steps with
$\Delta t=0.05$~s; arrival is declared within $0.3$~m of the goal.
All classes use $v_{\max}=1.0$~m/s. DI has
$a_{\max}=2.0$~m/s$^2$, and UNI and CL/FO have
$a_v=2.0$~m/s$^2$. We use $x_r=0.1$~m for UNI/DD and $x_r=0.2$~m for CL/FO, based on
their body geometry. For CL/FO, $L=0.5$~m,
$\phi_{\max}=40^\circ$, and $\dot{\phi}_{\max}=10$~rad/s.}



We use fixed CA-BRF parameters throughout the evaluation.
\footnote{We set $\lambda_1=\lambda_2=2.0$ and solve all
QPs using OSQP~\cite{stellato2020osqp}. For CA-BRF,
$d_{\rm act}=3.0$~m, $\lambda_B=2.0$, and $H=40$ steps, satisfying
$H\geq\bar K_{\rm stop}+1$. The interpolation grid is
$\Lambda=\{1,0.9,0.85,\ldots,0.05,0\}$.}
To isolate differences introduced by the safety mechanisms, all
methods use the same decentralized artificial potential field (APF)
nominal controller~\cite{khatib1986apf}. Its desired
operational-space velocity is converted to model-specific controls
through proportional tracking before each method applies its respective
safety mechanism.


\subsubsection*{\textbf{Evaluation Metrics}}

We evaluate navigation performance and safety using four metrics.
\texttt{Safe success rate} (\texttt{SR}, \%) is the percentage of
robots that reach their goals without violating the prescribed
pairwise safety condition.
\texttt{Arrival rate} (\texttt{AR}, \%) is the percentage of robots
that reach their goals within $T_{\max}$, regardless of safety
violations.
\texttt{Safety violation rate} (\texttt{SVR}, \%) is the percentage
of trials in which $h_{ij}<0$ for at least one pair at any control
update.
\texttt{Penetration depth} (\texttt{Pen.}, m) measures the severity
of these violations.
For each violating trial, we compute
$d_{\rm pen}:=\max_{i<j,k}
\max\!\left(0,R_{ij}-\|\mathbf p_i^k-\mathbf p_j^k\|\right)$
and report its mean over violating trials; \texttt{Pen.} is reported
as zero when no violation occurs.
We additionally report the mean number of distinct pairwise
safety-violation events per trial in
Table~\ref{tab:violation_events}.

\subsubsection*{\textbf{Comparison Methods}}
\label{subsubsec:comparison}


We compare CA-BRF with four methods.
\textbf{1) BRF-Shared} is an ablation that retains the HOCBF candidate
controller, admissibility checks, and backup-recoverability filter,
but uses a shared capability-agnostic operational-space braking model
instead of model-specific backup rollouts for recovery certification.
It isolates the effect of modeling each robot's recovery capability.
\textbf{2) HOCBF (w/o BRF)} uses the same operational-space HOCBF
candidate controller but directly applies its output, isolating the
effect of backup-recoverability filtering beyond instantaneous HOCBF
enforcement.
\textbf{3) Backup-CBF}~\cite{chen2020guaranteed_backup} implements a
decentralized backup-CBF controller with model-specific backup
trajectories. It constructs QP constraints from local linearizations
of backup-trajectory safety, whereas CA-BRF directly evaluates
successor backup recoverability.
\textbf{4) sRCBF}~\cite{kim2025srcbf} applies the smooth Robust CBF to
each robot's canonical representation, treating neighbors as
worst-case dynamic obstacles bounded by their maximum velocities.
Unlike CA-BRF, it models neither capability-dependent neighbor
recovery nor reciprocal avoidance. \vspace{-6pt}

\subsection{Random Start--Goal Navigation}
\label{subsec:navigation} \vspace{-3pt}

Tables~\ref{tab:main_results} and~\ref{tab:violation_events} summarize the navigation performance and pairwise safety-violation events. CA-BRF achieves the highest safe success rate (\texttt{SR}) at every team size, attaining $98.0\%$, $92.0\%$, and $85.6\%$ at $N=10$, $20$, and $30$, respectively. It incurs no safety violations over all $150$ trials while maintaining high arrival rates as the team size increases.

The performance gap widens in dense interactions. At $N=30$, BRF-Shared attains $74.9\%$ \texttt{SR} and $75.8\%$ \texttt{AR}, with a $28\%$ safety-violation rate, indicating that a shared recovery prediction becomes increasingly unreliable for coupled heterogeneous interactions. HOCBF (w/o BRF) attains a higher \texttt{AR} of $90.2\%$, $4.6$ percentage points above CA-BRF, but its \texttt{SR} falls to $79.5\%$, with an $86\%$ safety-violation rate and $3.70$ violation events per trial. The latter result shows that the HOCBF candidate controller alone does not provide the same closed-loop safety as explicitly preserving backup recoverability.



Backup-CBF likewise degrades with team size, reaching $58.1\%$
\texttt{SR}, $60.5\%$ \texttt{AR}, and $1.70$ violation events per
trial at $N=30$. This degradation indicates that locally linearized
backup-trajectory constraints can become inadequate in dense
interactions, supporting the benefit of directly evaluating successor
backup recoverability. sRCBF degrades most severely: its \texttt{SR}
decreases from $54.8\%$ at $N=10$ to $5.1\%$ at $N=30$, where it
yields $6.7\%$ \texttt{AR}, a $100\%$ safety-violation rate, and
$8.70$ violation events per trial. Treating neighbors solely as
worst-case dynamic obstacles is therefore overly restrictive in these
dense heterogeneous interactions. \vspace{-3pt}



\subsection{Effect of Capability-Aware Recovery Modeling}
\label{subsec:capability_recovery}

We next isolate the effect of capability-aware recovery modeling through a same-state certificate analysis. At states and candidate controls visited by CA-BRF, we reevaluate the same interpolation grid using model-specific and shared recovery rollouts, changing only the rollout used to assess recoverability. This removes differences caused by the methods visiting different closed-loop trajectories. An \emph{optimistic mismatch} occurs when BRF-Shared accepts a candidate fraction rejected by the model-specific test; a \emph{pessimistic mismatch} denotes the converse.

As shown in Table~\ref{tab:capability_analysis}, the shared recovery model changes the selected filter decision in $4.74\%$ of active components at $N=20$ and $11.61\%$ at $N=30$. The discrepancies are strongly asymmetric: BRF-Shared accepts fractions rejected by the model-specific test in $4.58\%$ and $11.25\%$ of components, whereas the reverse occurs in only $0.21\%$ and $0.48\%$, respectively.
The asymmetry is relevant to the stated guarantee. When BRF-Shared
accepts a fraction rejected by the model-specific test, its
shared-model certificate does not establish the recoverability
condition required in the recursive argument of
Theorem~\ref{thm:safety}. This does not itself imply a collision, but
Theorem~\ref{thm:safety} cannot be invoked for that decision.
\vspace{-15pt}



\begin{table}[t]
\centering
\caption{
Same-state comparison of model-specific and shared recovery
evaluation. All controller components are fixed except for the
backup rollout used to evaluate recoverability.
}
\label{tab:capability_analysis}
\setlength{\tabcolsep}{4.5pt}
\renewcommand{\arraystretch}{1.08}
\begin{tabular}{lcc}
\toprule
Metric & $N=20$ & $N=30$ \\
\midrule
Decision differs (\%)
& 4.74 & 11.61 \\
Shared accepts / CA-BRF rejects (\%)
& 4.58 & 11.25 \\
Shared rejects / CA-BRF accepts (\%)
& 0.21 & 0.48 \\
\bottomrule
\end{tabular} \vspace{-1.5em}
\end{table}

\subsection{Role of the Recovery Certificate}
\label{subsec:recovery_certificate}

We assess the effectiveness of the two terms, $B^h$ and $B^\psi$, in
the recovery certificate. We evaluate the two certificate definitions
along the same HOCBF-only trajectories, without applying
backup-recoverability filtering, and measure whether a certificate violation precedes a subsequent pairwise
safety-set violation. Using identical
trajectories isolates the additional warning information provided by
$B^\psi$.

In Table~\ref{tab:recovery_analysis}, using $B^h$ alone
detects only $67.0$--$76.9\%$ of subsequent safety violations, with
median lead times of $0.10$--$0.15$~s. Including $B^\psi$ increases
recall to $97.8$--$100.0\%$ and the median lead time to $0.85$~s
across both team sizes.
The two terms capture complementary aspects of recovery.
$B^h$ checks whether the prescribed backup trajectory maintains the
required geometric separation.
$B^\psi$ additionally accounts for closing motion along that
trajectory and can therefore indicate loss of the certified recovery
condition before the geometric margin itself becomes critical.
The higher recall and longer lead time support using both terms to
define $\mathcal C_{ij}^{B}$. \vspace{-15pt}

\begin{table}[t]
\centering
\caption{
Early-warning ability of the recovery certificate on HOCBF-only
trajectories.
Recall is the fraction of safety-buffer breaches preceded by a
certificate violation within the previous $40$ control steps
($2.0$\,s), and lead time is the median time from certificate
violation to breach.
}
\label{tab:recovery_analysis}
\setlength{\tabcolsep}{3.5pt}
\renewcommand{\arraystretch}{1.05}
\begin{tabular}{lcccc}
\toprule
& \multicolumn{2}{c}{$N=20$}
& \multicolumn{2}{c}{$N=30$} \\
\cmidrule(lr){2-3}\cmidrule(lr){4-5}
Certificate
& Recall (\%) & Lead (s)
& Recall (\%) & Lead (s) \\
\midrule
$B^h$ only
& 76.9 & 0.15
& 67.0 & 0.10 \\
$B^h$ and $B^\psi$
& 100.0 & 0.85
& 97.8 & 0.85 \\
\bottomrule
\end{tabular}  \vspace{-2em}
\end{table}

\subsection{Effect of Component-Wise Consistent Filtering}
\label{subsec:component_consensus} \vspace{-3pt}

We evaluate the need for a common interpolation factor
$\lambda_{\mathcal K}$ within each connected interaction component.
We compare CA-BRF with an independent-selection variant in which each
robot selects its largest locally accepted factor $\lambda_i$, with all
other components unchanged. Under this variant, robot $i$ evaluates
each incident pair assuming both robots apply $\lambda_i$. Thus, when
neighbors select $\lambda_i\neq\lambda_j$, the executed control pair
need not be the pair evaluated by either robot. We call this a
\emph{composition failure} when both local decisions are accepted but
their executed pair fails the backup-recoverability test; cases in
which either robot applies the full backup are excluded.

As shown in Fig.~\ref{fig:component_failure}, independent selection
produces composition failures in $13$ of $50$ trials at $N=20$ and
$26$ of $50$ trials at $N=30$, whereas CA-BRF produces none.
Same-state replay reproduces these failures, isolating their cause to
independently combined decisions instead of closed-loop trajectory
differences. Every observed active-pair transition from
$\mathcal C_{ij}^{B}$ under independent selection is preceded by a
composition failure; none is observed under CA-BRF.
A common $\lambda_{\mathcal K}$ thus assesses every active pair
under its executed control combination as required by
Theorem~\ref{thm:safety}. This consistency increases full-backup use:
at $N=30$, CA-BRF applies it in $8.6\%$ of active robot control steps,
versus $0.7\%$ under independent selection
(Fig.~\ref{fig:component_backup}). The additional intervention
eliminates the observed composition failures and preserves a jointly
valid recovery certificate across coupled interactions. \vspace{-1pt}


\begin{figure}[!t]
    \centering
    \captionsetup[subfigure]{skip=-2pt}

    \includegraphics[width=0.5\columnwidth]
    {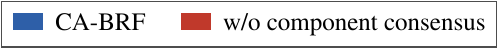}
    \vspace{0.5em}
    
    \begin{subfigure}[t]{0.49\columnwidth}
        \centering
        \includegraphics[width=\linewidth]
        {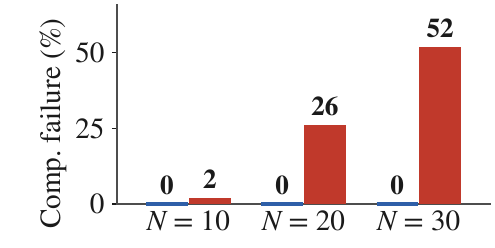}
        \caption{Inconsistent local decisions.}
        \label{fig:component_failure}
    \end{subfigure}
    \begin{subfigure}[t]{0.49\columnwidth}
        \centering
        \includegraphics[width=\linewidth]
        {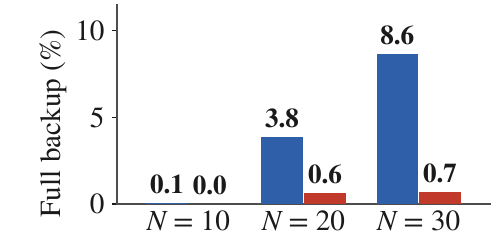}
        \caption{Full-backup activation.}
        \label{fig:component_backup}
    \end{subfigure}

    \caption{Effect of coordinating filter decisions across simultaneously interacting robots. CA-BRF uses one common candidate--backup interpolation factor within each connected interaction component, whereas the comparison variant selects factors independently. (a) Fraction of trials with composition failures, in which locally accepted decisions fail the pairwise backup-recoverability test when executed together. (b) Fraction of active robot control steps applying the full backup ($\lambda=0$). Coordinated selection removes the observed composition failures while requiring more frequent full-backup activation in dense interactions. }
    \label{fig:component_consensus} \vspace{-1.5em}
\end{figure} \vspace{-.2cm}

\subsection{Physical Robot Validation}
\label{subsec:physical}
\vspace{-3pt}

We conduct physical robot experiments with five LIMO Pro
robots~\cite{limopro}, configured to implement the five kinematic
models considered in this work (DI, UNI, DD, CL, and FO), to
demonstrate the applicability of CA-BRF in real-robot navigation.
We use $v_{\max}=0.3$~m/s. As shown in Fig.~\ref{fig:overview}, the robots are initialized symmetrically with goals on opposite sides of
the workspace, requiring passage through a shared central region. We
compare CA-BRF with HOCBF (w/o BRF) to assess whether
backup-recoverability filtering maintains the prescribed safety
condition in real-robot navigation. CA-BRF guides all robots to their
goals without collision, whereas directly applying the HOCBF candidate
controls causes a collision, demonstrating the filter's role in
preserving safe interactions.
 \vspace{-3pt}

\vspace{-.1cm}
\section{Conclusion}
\vspace{-.1cm}

This paper presented CA-BRF, a capability-aware
backup-recoverability filter for safe heterogeneous multi-robot
navigation. CA-BRF certifies model- and state-dependent admissible
backup maneuvers and consistently filters candidate controls across
simultaneous interactions. Under the stated conditions, it is
recursively feasible, applies admissible controls, and preserves
pairwise safety at every control update. Simulations with up to $30$
heterogeneous robots and physical-robot experiments demonstrate the
benefits of capability-aware recovery and consistent filtering.
Several limitations remain.
The guarantee is limited to discrete control-update instants; new
robot models require an admissible backup policy with a finite
stopping-time bound, and deadlock-free progress is not guaranteed.
Future work will address continuous-time safety, online model
adaptation, and deadlock-aware coordination. \vspace{-3pt}. \vspace{-3pt}



%
\bibliographystyle{IEEEtran} \footnotesize
\bibliography{references}

\end{document}